\documentclass{new_tlp}
\usepackage{mathptmx}

\usepackage{xspace}
\usepackage{multirow}
\usepackage{graphics}
\usepackage{lscape}
\usepackage{url}

\def\sm{{SM}\xspace}
\def\smt{{SMT}\xspace}
\def\sr{{SR}\xspace}
\def\sri{{SRI}\xspace}
\def\srt{{SRT}\xspace}
\def\srti{{SRTI}\xspace}

\def\srtiasp{{\sc SRTI-ASP}\xspace}
\def\sricp{{\sc SRI-CP}\xspace}
\def\sraf{{\sc SR-AF}\xspace}
\def\sriaf{{\sc SRI-AF}\xspace}

\def\clingo{{\sc Clingo}\xspace}
\def\choco{{\sc Choco}\xspace}
\def\lar{\leftarrow}
\def\ba{\begin{array}}
\def\ea{\end{array}}
\def\beq{\begin{equation}}
\def\eeq#1{\label{#1}\end{equation}}
\def\no{\ii{not}}
\def\ii#1{\hbox{\it #1\/}}
\def\is#1{\hbox{\scriptsize\it #1\/}}

\newtheorem{theorem}{Theorem}

\title[Theory and Practice of Logic Programming]
        {A General Framework for Stable Roommates Problems \\ using Answer Set Programming\thanks{This work has been partially supported by Sabanci University IRP Grant, and the Scottish Informatics and Computer Science Alliance DVF Programme. The third and fourth authors were supported by grants EP/P028306/1 and EP/P026842/1 from the Engineering and Physical Sciences Research Council, respectively.}}

\author[Erdem, Fidan, Manlove, Prosser]{
        ESRA ERDEM and M\"UGE FIDAN\\
        Faculty of Engineering and Natural Sciences, Sabanci University, Istanbul, Turkey \\
        \{esraerdem,mugefidan\}@sabanciuniv.edu
        \and
        DAVID MANLOVE and PATRICK PROSSER\\
        School of Computing Science, University of Glasgow, Glasgow, UK\\
        \{david.manlove,patrick.prosser\}@glasgow.ac.uk }

\begin{document}

\label{firstpage}

\maketitle

\begin{abstract}
The Stable Roommates problem (SR) is characterized by the preferences of agents over other agents as roommates: each agent ranks all others in strict order of preference. A solution to SR is then a partition of the agents into pairs so that each pair shares a room, and there is no pair of agents that would block this matching (i.e., who prefers the other to their roommate in the matching). There are interesting variations of SR that are motivated by applications (e.g., the preference lists may be incomplete (SRI) and involve ties (SRTI)), and that try to find a more fair solution (e.g., Egalitarian SR). Unlike the Stable Marriage problem, every SR instance is not guaranteed to have a solution. For that reason, there are also variations of SR that try to find a good-enough solution (e.g., Almost SR). Most of these variations are NP-hard. We introduce a formal framework, called SRTI-ASP, utilizing the logic programming paradigm Answer Set Programming, that is provable and general enough to solve many of such variations of SR. Our empirical analysis shows that SRTI-ASP is also promising for applications.  This paper is under consideration for acceptance in TPLP.
\end{abstract}


\section{Introduction} \label{sec:intro}

The Stable Roommates problem~\cite{galeShapley1962} (\sr) is a matching problem (well-studied in Economics and Game Theory) characterized by the preferences of an even number $n$ of agents over other agents as roommates: each agent ranks all others in strict order of preference. A solution to \sr is then a partition of the agents into pairs that are {\em acceptable} to each other (i.e., they are in the preference lists of each other), and the matching is {\em stable} (i.e., there exist no two agents who prefer each other to their roommates, and thus {\em block} the matching).

\sr is an interesting computational problem, not only due to its applications (e.g., for pairing in large-scale chess competitions~\cite{kujansuuLM1999}, for campus house allocation~\cite{arkin2009}, pairwise kidney exchange~\cite{ROTH2005},
creating partnerships in P2P networks~\cite{Gai2007}) but also due to its computational properties described below.

\smallskip\noindent {\em Incomplete preference lists with ties.}
Upon a question posed by Knuth~\citeyear{knuth76} in 1976 about the existence of an algorithm for \sr, Irving~\citeyear{irving1985} developed a linear-time algorithm for \sr. Meanwhile, researchers have started investigating variations of \sr motivated by further observations and applications. For instance, in practice (like large-scale chess tournaments), agents may find it difficult to rank a large number of alternatives in strict order of preference. With such motivations, \sr has been studied with incomplete preference lists (\sri)~\cite{gusfield89}, with preference lists including ties (\srt)~\cite{RONN1990}, and with incomplete preference lists including ties (\srti)~\cite{irving2002}.
Interestingly, some of these slight variations (i.e., the existence of a stable matching in \srt and \srti) are proven to be NP-complete (Table~\ref{tab:complexity}).

\smallskip\noindent{\em Stable and more fair solutions.}
With the motivation of finding more fair stable solutions, variations of \sr have been studied. For instance, Egalitarian \sr aims to maximize the total satisfaction of preferences of all agents; it is NP-hard~\cite{feder1992}.
Rank Maximal \sri aims to maximize the number of agents matched with their first preference, and then, subject to this condition, to maximize the number of agents matched with their second preference, and so on; it is also NP-hard~\cite{Cooper2020}.

\smallskip \noindent {\em Not stable but good-enough solutions.}
As first noted by Gale and Shapley~\citeyear{galeShapley1962}, unlike the Stable Marriage problem (\sm), there is no guarantee to find a solution to every \sr problem instance (i.e., there might be no stable matching).  When an \sr instance does not have a stable solution, variations of \sr have been studied to find a good-enough solution. Almost \sr aims to minimize the total number of blocking pairs (i.e., pairs of agents who prefer each other to their roommates); it is NP-hard~\cite{abraham2005}.

\smallskip
Alongside these interests in \sr, some exact methods and software \cite{sricp2020,matwa2020} have been developed to solve \sr and \sri (both solvable in poly-time) using Constraint Programming (CP)~\cite{Prosser2014}, and based on Irving's algorithm~\cite{irving1985}.
However, to the best of the authors' knowledge, there is no exact method (except for the enumeration based method for Egalitarian \sri) and implementation, that provides a solution to any intractable variation of \sr, described in three groups above.

\smallskip \noindent {\em Our Contributions.} We introduce a formal framework and its implementation, called \srtiasp, that are general enough to provide solutions to all variations of \sr mentioned above, including the intractable decision/optimization versions: \srt, \srti, Egalitarian \srti, Rank Maximal \srti, Almost \srti. \srtiasp provides a flexible framework to study variations of \sr.

\srtiasp utilizes a logic programming paradigm, called Answer Set Programming (ASP) \cite{BrewkaEL16}, to declaratively solve stable roommates problems.
We represent \sri and its variations in the expressive formalism of ASP, and \srtiasp computes models of these formulations using the ASP solver~\clingo~\cite{GebserKKOSS11}.
For each variation of \sr, given a problem instance, \srtiasp returns a solution (or all solutions) if one exists; otherwise, it returns that the problem does not have a solution.  We prove that \srtiasp is sound and complete (Theorem~\ref{thm:correct}).

We have evaluated \srtiasp over different sizes of \sri instances (randomly generated with the software~\cite{sricp2020}, called \sricp from now on) to understand its scalability, as the input size, and the degree of completeness of preference lists increase. We have developed a method to add ties to these instances, and empirically analyzed the scalability of \srtiasp on \srti instances as well.

We have compared \srtiasp with \sricp, over \sri instances. We have also investigated the use of \sricp to solve Egalitarian \sri, Rank Maximal \sri and Almost \sri based on enumeration-based brute-force methods, and compared \srtiasp with these methods.

In addition, we have compared \srtiasp with the ASP method proposed by Amendola~\citeyear{Amendola18} for \sr (called \sraf from now on) based on Argumentation Framework (AF)~\cite{DUNG1995}, over \sr instances.

           \begin{table}[ht]
                \centering
               \caption{Summary of the complexities of \sr problems}
            	\begin{tabular}{ll}
            		\hline \hline
            		Problem & Complexity \\
            		\hline
                    \sr & P \cite{irving1985}\\
                    \sri & P \cite{gusfield89} \\
                    \srti (super) & P \cite{irving2002} \\
                    \srti (strong) & P \cite{kunysz2016,scott2005} \\
                    \srt (weak) (and thus \srti (weak)) & NP-complete \cite[Thm 1.1, Prop 2.2]{RONN1990} \\ 
                    \srti (weak) & NP-complete$^{*}$ \cite[Thm 5]{Irving2009} \\ 
                    Egalitarian \sr & NP-hard \cite[Thm 8.3]{feder1992} \\
            	    Egalitarian \sri & NP-hard$^{*}$ \cite[Cor 4]{Cseh2019} \\
            		Almost \sr (and thus \srt (weak))  & NP-hard \cite[Thm 1]{abraham2005} \\
            	    Almost \sri  & NP-hard$^{*}$ \cite[Thm 1]{Biro2012} \\
  				\hline \hline          	
            	\end{tabular}
            	\footnotesize{$^{*}$ for short lists of size $\leq$ 3}  \\

                \vspace{-.7\baselineskip}
            	\label{tab:complexity}
                \vspace{-.75\baselineskip}
            \end{table}

\section{Stable Roommates Problems}\label{sec:sr}

Let us start with defining the Stable Roommate problem with Incomplete lists (\sri).
Let $A$ be a finite set of agents. For every agent $x\in A$, let $\prec_{x}$ be a strict and total ordering of preferences over a subset $A_x$ of $A \backslash \{x\}$.  We refer to $\prec_{x}$ as agent $x$'s preference list. For two agents $y$ and $z$, we denote by $y \prec_{x} z$ that $x$ prefers $y$ to $z$. Since the ordering of preferences is strict and total, for every agent $x\in A$, for every two distinct agents $y$ and $z$ in $A_x$, either $y \prec_{x} z$ or $z \prec_{x} y$. Note that the preferences of agents with respect to $\prec_{x}$ are transitive and asymmetric. If an agent $x$ is in $y$'s preference list, then $x$ is called {\em acceptable} to $y$. We denote by $\prec$ the collection of all preference lists.

A \emph{matching} for a given \sri instance is a function
${M: A \mapsto A}$ such that, for all $\{x,y\} \subseteq A$ such that $x\in A_y$ and $y\in A_x$, $M(x)=y$ if and only if $M(y)=x$. If agent $x$ is mapped to itself, we then say he/she is \emph{single}.

A matching $M$ is {\em blocked} by a pair $\{x, y\} \subseteq A$ ($x\neq y$) if 
\begin{itemize}
	\item[B1] both agents $x$ and $y$ are acceptable to each other,
    \item[B2] $x$ is single with respect to $M$, or $y \prec_{x} M(x)$, and
    \item[B3] $y$ is single with respect to $M$, or $x \prec_{y} M(y)$.
\end{itemize}
A matching for \sri is called {\em stable} if it is not blocked by any pair of agents.
Fig.~\ref{fig:sri7-4-8} illustrates three examples for \sri.

\begin{figure}[h]
\centering \vspace{-.5\baselineskip}
\resizebox{0.9\textwidth}{!}{\begin{tabular}{ccc}
\hline \hline
sri7 & sri4 & sri8 \\
\hline
$
\ba l
		a:\ b\ e\ d\ f\ g \\		
		b:\ c\ f\ a\ g\ e \\		
	    c:\ d\ g\ b\ e\ f\ a \\		
		d:\ a\ c\ e\ f\ g  \\		
		e:\ f\ a\ b\ c\ d \\		
		f:\ g\ b\ e\ c\ d\  a \\		
		g:\ c\ f\ d\ a\ b \\
\ea
$
&
$
\ba l
a:\ b\ c\ d \\
b:\ c\ a\ d \\
c:\ a\ b\ d \\
d:\ a\ b\ c \\
\ea
$
&
$
\ba l
		a:\ c\ e\ f\ g\ d\ h \\
		b:\ d\ f\ h\ c\ g \\	
		c:\ a\ b\ f\ h\ e\ d \\		
		d:\ h\ g\ e\ a\ b\ c \\		
		e:\ g\ c\ b\ d\ a\ f\\		
		f:\ e\ a\ g\ c\ h\ b\\		
		g:\ f\ h\ d\ b\ c \\		
		h:\ b\ d\ a\ e\ f \\				
\ea
$ \\
\hline \hline
\end{tabular}} \vspace{-.5\baselineskip}
\caption{Three \sri instances. sri7 has a single stable solution $\{\{a,b\}, \{c,d\}, \{f,g\}, \{e\}\}$. sri4 has no solutions, since each possible matching is blocked (e.g., $\{\{a,c\}, \{b,d\}\}$ is blocked by $\{a,b\}$). sri8 has two stable matchings, $M_1{:}\{\{a,c\}, \{b,h\}, \{d,e\}, \{f,g\}\}$ and $M_2{:}\{\{a,c\}, \{b,h\}, \{d,g\}, \{e,f\}\}$.}
\label{fig:sri7-4-8}
 \vspace{-.75\baselineskip}
\end{figure}

{\em The Stable Roommates problem (\sr)} is a special case of \sri where the preference orderings are strict and complete
(i.e., for every agent $x \in A$, $A_x {=} A \setminus \{x\}$), and $|A|$ is even.

\smallskip
\noindent {\em Ties.}
{\em The Stable Roommates problem with Ties and Incomplete Lists (\srti)} is a variation of \sri where the preference lists are partial orderings and where incomparability is transitive. In this context, ties correspond to indifference in the preference lists: an agent $x$ is {\em indifferent} between the agents $y$ and $z$, denoted by $y \sim_{x} z$, if $y \not \prec_{x} z$ and $z \not \prec_{x} y$. There are three levels of stability~\cite{irving2002}: {\em weak stability}, {\em strong stability}, and {\em super stability}. We will focus on {\em weak stability} in this paper,  since it is a harder problem compared to the other two versions (Table~\ref{tab:complexity}). Relative to weak stability, a pair $\{x, y\}$ of agents {\em blocks} a matching $M$ if conditions B1--B3 hold.

{\em The Stable Roommates problem with Ties (\srt)} is a special case of \srti where the preference ordering of each agent $x$ is over $A\setminus \{x\}$ and complete, and $|A|$ is even.

Note that, while the problems \sr and \sri are in P, \srt and \srti under weak stability are NP-complete (Table~\ref{tab:complexity}).

\smallskip \noindent {\em Fairness.}  When an \sri instance has many stable matchings, it may be useful to identify a stable matching that is fair to all agents. Different fairness criteria on top of stability have led to optimization variations of \sri.

Let $\mathcal{M}$ denote the set of all stable matchings of a given \sri instance $(A, \prec)$. For every agent $x$ and every agent $y\in A_x$, let $\ii{rank}(x,y)$ denote the rank of agent $y$ in the preference list $A_x$ of agent~$x$. We assume that agents prefer matching with a roommate: for every agent $x$, let $\ii{rank}(x,x)$ be a number larger than $\ii{rank}(x,y)$ for every $y\in A_x$.

{\em Egalitarian \sri} aims to maximize the total satisfaction of preferences of all agents. Let $M$ be a matching. For every agent $x$, we define the satisfaction $c_M(x)$ of $x$'s preferences with respect to $M$ as follows: $c_M(x) {=} I$ if $\ii{rank}(x,M(x)){=}I$. Then the total satisfaction of preferences of all agents is defined as follows: $c(M) = \sum_{x\in A} c_M(x).$ Note that for \sri, all matching $M$ have the same number of contributions of $\ii{rank}(x,x)$ values to $c(M)$.
Since the preferred agents have lower rankings, the total satisfaction of preferences of all agents is maximized when $c(M)$ is minimized.  Then, a matching $M {\in} \mathcal{M}$ with the minimum $c(M)$ is {\em egalitarian.}

{\em Rank Maximal \sri} considers different fairness criterion: it aims to maximize the number of agents matched with their first preferences, and then, subject to this condition, to maximize the number of agents matched with their second preference, and so on.
We start with the set $\mathcal{M}$ of all matchings of a given \sri instance $(A, \prec)$, and define a series of subsets $\mathcal{M}_{max}(i)$ of these matchings where the maximum number of agents are matched with their $i$'th preferences:
$$
\ba l
\hspace{-1ex}\mathcal{M}_{max}(0) = \mathcal{M} \\
\hspace{-1ex}\mathcal{M}_{max}(i) = \{M {\in} \mathcal{M}_{max}(i{-}1): 1 {\leq} i {\leq} |A|{-}1,\ \forall\ M' {\in} \mathcal{M}_{max}(i{-}1)\ \textrm{s.t.}\ M' {\neq} M\\
\qquad |\{x{\in}A: \ii{rank}(x{,}M(x)){=}i, x \neq M(x)\}| \geq |\{x{\in}A: \ii{rank}(x{,}M'(x)){=}i, x \neq M(x)\}| \}.
\ea
$$
Then, a matching $M \in \mathcal{M}_{max}(|A|-1)$ is {\em rank-maximal}.

Consider the \sri instance sri8 illustrated in Fig.~\ref{fig:sri7-4-8}, with two stable matchings.
Stable matching $M_1$ is egalitarian
and $M_2$ is rank-maximal.

\smallskip \noindent {\em Almost stable.}
Unlike the Stable Marriage problem, there is no guarantee to find a solution to every \sri problem instance (cf. sri4 in Fig.~\ref{fig:sri7-4-8}).  When an \sri instance does not have a stable matching, further variations of \sri have been studied to find a good-enough solution.

{\em Almost \sri} aims to minimize the total number of blocking pairs.
Let $bp_M(x,y)$ denote the set of blocking pairs of a given matching $M$. A matching $M{\in}\mathcal{M}$ is {\em almost stable} if it is blocked by the minimum number $|bp_M(x,y)|$ of pairs.


\section{Answer Set Programming}\label{sec:asp}

\srtiasp utilizes Answer Set Programming (ASP)~\cite{BrewkaEL16} to declaratively solve stable roommates problems. The idea of problem solving with ASP is (1) to represent the given problem by a program whose answer sets~\cite{GelfondL88,GelfondL91} characterize the solutions of the problem, and (2) to solve the problem using answer set solvers, like \clingo~\cite{GebserKKOSS11}.

{\em Why ASP?} We use ASP as an underlying paradigm for modeling and solving stable roommates problems for the following reasons. (1) Deciding whether a program in ASP has an answer set is NP-complete~\cite{dantsin2001}, so ASP is expressive enough for solving hard \sr problems. (2) ASP has expressive languages with a rich set of utilities, such as nondeterministic choices, hard constraints, weighted weak constraints with priorities,
and thus allow us to easily formulate different variations of \sr. (3) Efficient ASP solvers, like \clingo, supports these utilities. (4) Such an elaboration tolerant~\cite{mccarthy98} representation framework and flexible software environment are useful in studying and understanding variations of \sr in different applications. (5) Due to declarative problem solving in the formal framework of ASP, we can easily prove the soundness and completeness of \srtiasp (see Theorem~\ref{thm:correct}).

{\em Programs in ASP} Let us briefly describe the syntax of programs and useful constructs used in the paper.
We consider ASP programs that consist of rules of the form
\beq
\ii{Head} \lar A_1, \dots, A_m, \no\ A_{m+1}, \dots, \no\ A_n.
\eeq{eq:rule}
where $n \geq m \geq 0$, \ii{Head} is an atom or~$\bot$, and each $A_i$ is an atom.
A rule is called a
\textit{fact} if $m=n=0$ and a \textit{(hard) constraint} if \ii{Head} is~$\bot$.

{\em Cardinality expressions} are special constructs of the form
$l \{A_1,\dots,A_k\} u$
where each $A_i$ is an atom and $l$ and $u$ are
nonnegative integers denoting the lower and upper bounds~\cite{simons02}.
Programs using these constructs can be viewed as abbreviations for programs that consist
of rules of the form~(\ref{eq:rule}). Such an expression describes the subsets of the set
$\{A_1,\dots,A_k\}$ whose cardinalities are at least $l$ and at most~$u$.
Cardinality expressions can be used in heads of rules; then they generate many answer sets whose cardinality is
at least $l$ and at most $u$.

{\em Schematic variables} A group of rules that follow a pattern can be often described in a
compact way using ``schematic variables''. For instance,
the cardinality expression $1\{p_1,\dots,p_7\}1$ can be represented as
$1\{p(i):\ii{index}(i)\}1$, along with a definition of $\ii{index}(i)$ that describes the ranges of variables: $\ii{index}(1..7)$.

{\em Weighted weak constraints with priorities} The ASP programs can be augmented with ``weak constraints''---expressions of the following form~\cite{BuccafurriLR00}:
$$\mathop{\leftarrow}^{\sim} \ii{Body}(t_1,...,t_n) [w@p,t_1,...,t_n] .$$
Here, $\ii{Body}(t_1,...,t_n)$ is a formula (as in the body of a rule) with the terms $t_1,...,t_n$. Intuitively, whenever an answer set for a program satisfies $\ii{Body}(t_1,...,t_n)$, the tuple $\langle t_1,...,t_n \rangle$ contributes a cost of $w$ to the total cost function of priority $p$. The ASP solver tries to find an answer set with the minimum total cost. For instance, the following weak constraint
$$
\mathop{\leftarrow}^{\sim} p(i), p(i+1), index(i), index(i+1) [1@2,i]
$$
instructs \clingo\ to compute an answer set that does not include both $p(i)$ and $p(i+1)$, if possible. However, if \clingo\ cannot find such an answer set, it is allowed to compute an answer set with these atoms $p(i)$ and $p(i+1)$ but with an additional cost of 1 per each such $i$. Weak constraints are considered by \clingo\ according to their priorities.

\section{Solving \sri using ASP}\label{sec:sriasp}

We formalize the input $I=(A,\prec)$ of an \sri instance in ASP by a set $F_I$ of facts using atoms of the forms $\ii{agent}(x)$ (``$x$ is an agent in $A$'') and $\ii{prefer2}(x,y,z)$ (``agent $x$ prefers agent $y$ to agent $z$, i.e., $y \prec_{x} z$ ''). For instance, the preference list of agent $a$ in sri4 of Fig.~\ref{fig:sri7-4-8} is described by the following facts:
$\ii{prefer2}(a,b,c).\ \ii{prefer2}(a,c,d).$

For every agent $x$, since $x$ prefers being matched with a roommate $y$ in $A_x$ instead of being single, for every $y\in A_x$, we also add facts of the form $\ii{prefer2}(x,y,x)$. For the example above, the input also includes the facts:
$\ii{prefer2}(a,b,a).\ \ii{prefer2}(a,c,a).\ \ii{prefer2}(a,d,a). $

In the ASP formulation $P$ of \sri, the variables $x$, $y$, $z$ and $w$ denote agents in $A$. The program~$P$ starts with the definition of preferences of agents with respect to $\prec_{x}$:
\beq
\ba l
\ii{prefer}(x,y,z) \lar \ii{prefer2}(x,y,z). \\
\ii{prefer}(x,y,z) \lar \ii{prefer2}(x,y,w), \ii{prefer}(x,w,z).
\ea
\eeq{eq:P-pref}
The first rule expresses that being single is the least preferred option.
The second rule expresses that the preference relation is transitive.

Based on the preferences of agents, we define the concept of acceptability for each agent:
\beq
\ba l
\ii{accept}(x,y) \lar \ii{prefer}(x,y,\_). \\
\ii{accept}(x,y) \lar \ii{prefer}(x,\_,y).
\ea
\eeq{eq:P-accept}
and the concept of mutual acceptability:
\beq
\ii{accept2}(x,y) \lar \ii{accept}(x,y), \ii{accept}(y,x).
\eeq{eq:P-accept2}

The output $M: A \mapsto A$ of an \sri instance is characterized by atoms of the form $\ii{room}(x,y)$ (``agents $x$ and $y$ are roommates'').  The ASP formulation $P$ of \sri first generates pairs of roommates. For every agent $x$, exactly one mutual acceptable agent $y$ is nondeterministically chosen as $M(x)$ by the choice rules:
\beq
\ba l
1\{\ii{room}(x,y){:} \ii{agent}(y), \ii{accept2}(x,y)\}1 \lar \ii{agent}(x).
\ea
\eeq{eq:P-room}
Here, the roommate relation is symmetric:
\beq
\lar \ii{room}(x,y), \no\ \ii{room}(y,x).
\eeq{eq:P-room2}
The agents who are not matched with a roommate are single agents:
\beq
\ii{single}(x) \lar \ii{room}(x,x).
\eeq{eq:P-single}

Then, the stability of the generated matching is ensured by the hard constraints:
\beq
\lar \ii{block}(x,y) \qquad (x\neq y) .
\eeq{eq:P-stable}
Here, atoms of the form $\ii{block}(x,y)$ describe the blocking pairs (i.e., conditions B1--B3):
\beq
\ba l
\ii{block}(x,y) \lar \ii{accept2}(x,y), \ii{single}(x), \ii{single}(y), \no\ \ii{room}(x,y).  \quad (x\neq y)\\
\ii{block}(x,y) \lar \ii{accept2}(x,y), \ii{single}(x), \ii{like}(y,x), \no\ \ii{room}(x,y).  \quad (x\neq y) \\
\ii{block}(x,y) \lar \ii{accept2}(x,y), \ii{like}(x,y), \ii{single}(y), \no\ \ii{room}(x,y).  \quad (x\neq y) \\
\ii{block}(x,y) \lar \ii{accept2}(x,y), \ii{like}(x,y), \ii{like}(y,x), \no\ \ii{room}(x,y).  \quad (x\neq y) \\
\ea
\eeq{eq:P-block}
where $x{\neq}y$ in each rule, and atoms $\ii{like}(x,y)$ describe that agent $x$ prefers agent $y$ to her/his roommate~$x'=M(x)$:
\beq
\ii{like}(x,y) \lar \ii{room}(x,x'), \ii{prefer}(x,y,x'). \quad (x'\neq y)
\eeq{eq:P-like}

Given the ASP formulation $P$ whose rules are described above and the ASP description $F_I$ of an \sri instance $I$, the ASP solver \clingo generates a stable matching (or all stable matchings), if one exists; otherwise, it returns that there is no solution.
This is possible since the ASP program~$P$ (i.e., (\ref{eq:P-pref})--(\ref{eq:P-like})) is sound and complete.

\newcommand{\correctness}{
Given an \sri instance $I=(A,\prec)$, for each answer set~$S$ for
$P \cup F_I$, the set of atoms of the form $\ii{room}(x,y)$ in
$S$ encodes a stable matching $M: A\mapsto A$ to the \sri
problem instance. Conversely, each stable matching for the given \sri instance
corresponds to a single answer set for $P \cup F_I$.
}

\begin{theorem}\label{thm:correct}
\correctness
\end{theorem}


\begin{proof}
First, we show that the answer set for (\ref{eq:P-pref})--(\ref{eq:P-accept2}) correctly describes the acceptability relation.
\begin{enumerate}[(i)]
\item
Due to Proposition~4 of Erdem and Lifschitz~\citeyear{ErdemL03} about the correctness of the transitive closure definition, the answer set $X_0$ for (\ref{eq:P-pref}) correctly defines preferences of each agent $x$ (by means of atoms of the $\ii{prefer}(x,y,z)$).
\item
Due to Proposition~3 of Erdogan and Lifschitz~\citeyear{Erdogan2004}, adding (\ref{eq:P-accept}) to (\ref{eq:P-pref}) conservatively extends $X_0$ to $X_1$, which also describes the preference lists for each agent $x$ (by means of atoms of the $\ii{accept}(x,y)$).
\item
Due to Proposition~3 of Erdogan and Lifschitz~\citeyear{Erdogan2004}, adding (\ref{eq:P-accept2}) to $(\ref{eq:P-pref}) \cup (\ref{eq:P-accept})$ conservatively extends $X_1$ to $X$.
\item
The answer set $X$ describes the acceptability of every pair of agents $x$ and $y$ to each other (by means of atoms of the $\ii{accept2}(x,y)$).
\end{enumerate}

Next, we show that the answer set for (\ref{eq:P-pref})--(\ref{eq:P-room2}) correctly characterizes a matching.
\begin{enumerate}[(i)]
\item
We use the splitting set theorem~\cite{Erdogan2004,LifschitzT94}: Let $\Pi$ be the program (\ref{eq:P-pref})--(\ref{eq:P-room}), and the splitting set $U$ be the set of atoms of the form $\ii{prefer}(x,y,z)$, $\ii{accept}(x,y)$, and $\ii{accept2}(x,y)$.
\item
The bottom $b_U(\Pi)$ consists of (\ref{eq:P-pref})--(\ref{eq:P-accept2}), and the top part consists of rules (\ref{eq:P-room}).
\item
Every answer set $Y$ for the top part  $e_U(\Pi \setminus b_U(\Pi), X)$ evaluated with respect to $X$, describes a function, via atoms of the form $\ii{room}(x,y)$, which maps every agent $x$ to exactly one agent $y$ so that $x$ and $y$ are acceptable to each other (when $\ii{accept2}(x,y)\in X$).
Moreover, every such mapping can be characterized by a unique answer set for $e_U(\Pi \setminus b_U(\Pi), X)$.
\item
According to the splitting set theorem, $X\cup Y$ is an answer set for (\ref{eq:P-pref})--(\ref{eq:P-room}).
\item
Then, $X\cup Y$ describes a mapping between acceptable pairs of agents. The symmetry of this mapping is guaranteed by (\ref{eq:P-room2}), using Proposition~2 of Erdogan and Lifschitz~\citeyear{Erdogan2004}, leading to a matching.
\item
Therefore, there is a one-to-one correspondence between the answer sets for (\ref{eq:P-pref})--(\ref{eq:P-room2}) and the matchings between acceptable pairs of agents.
\end{enumerate}

Next, using Proposition~3 of Erdogan and Lifschitz~\citeyear{Erdogan2004} three times, we show that adding definitions~(\ref{eq:P-single}), (\ref{eq:P-like}), and (\ref{eq:P-block}) to (\ref{eq:P-pref})--(\ref{eq:P-room2}) one by one conservatively extends the answer sets for (\ref{eq:P-pref})--(\ref{eq:P-room2}), by describing singles (by means of atoms of the $\ii{single}(x)$), preferences $y$ of every agent $x$ over her/his roommate (by means of atoms of the $\ii{like}(x,y,z)$), and then blocking pairs (by means of atoms of the $\ii{block}(x,y)$).

Finally, we show that stability is guaranteed by (\ref{eq:P-stable}), using Proposition~2 of Erdogan and Lifschitz~\citeyear{Erdogan2004}. Therefore, there is a one-to-one correspondence between every answer set $S$ for (\ref{eq:P-pref})--(\ref{eq:P-stable}) and every stable matching.
\end{proof}

\noindent {\em Ties (weak stability).}
As noted by Cseh~\citeyear{Cseh2019}, relative to weak stability,
stability in \srti instances can be defined in exactly the same way as for \sri.
Therefore, we can use the \sri formulation $P$ to solve \srti instances too.

\smallskip \noindent {\em Fairness.}
Let us describe the ranks of agents by a set of facts using atoms of the form $\ii{rank}(x,y,i)$ (``the rank of agent $y$ according to agent $x$'s preferences is $i$''). For each agent $x$, since $x$'s preference ordering $\prec_x$ is total,
we can define the ranks as follows: the $i$'th agent in the preference list $A_x$ of $x$ has rank $i$, and $x$ has a rank larger than the ranks of $i$.
$$
\ba l
\ii{rank}(x,b,1) \lar \#\ii{count}\{a:\ \ii{prefer}(x,a,b), a\neq b\}{=}0, \ii{accept}(x,b) \\ 
\ii{rank}(x,b,i) \lar \ii{rank}(x,a,i{-}1), \ii{prefer}(x,a,b), \\
\qquad \#\ii{count}\{c: \ii{prefer}(x,a,c), \ii{prefer}(x,c,b), c\neq a, c\neq b\}{=}0 \quad (a\neq b, i>1) 
\ea
$$

Egalitarian \sri aims to maximize the total satisfaction of preferences of all agents by a matching $M$. The satisfaction $c_M(x)$ of an agent $x$'s preferences with respect to $M$ is defined as the rank of $M(x)$.
Since more preferred agents have lower ranks, the total satisfaction of preferences of all agents is maximized when $c(M){=}\sum_{x\in A} c_M(x)$ is minimized. Therefore, to solve Egalitarian \sri, we simply add to the \sri formulation $P$, the weighted weak constraints:
$$
\mathrel{\mathop{\lar}^{\sim}}  \ii{room}(x,y), \ii{rank}(x,y,r).\  [r@1,x]  
$$
which instruct \clingo to minimize the sum of the ranks $r$ of roommates.

Rank Maximal \sri tries to maximize the number of agents matched with their first preferences, and, subject to this condition, tries to maximize the number of agents matched with their second preferences, and so on.  Such an iterative definition can be modeled elegantly by the following weak constraints:
$$
\mathrel{\mathop{\lar}^{\sim}} \ii{room}(x,y), \ii{rank}(x,y,r).\ [-1@|A|-r,x,y] \qquad (x\neq y)
$$
Note that the priorities of these weak constraints are defined as $|A|-r$ for every pair of roommates $\{x,y\}$ ($x\neq y$), where $\ii{rank}(x,y)=r$.
As the rank changes from $1$ to $|A|-1$, the priority decreases. The ASP solver \clingo handles weak constraints with respect to their priorities. On the other hand, note that the weights of these weak constraints are specified as -1 for every pair of roommates $\{x,y\}$ ($x\neq y$). Therefore,  \clingo first considers the highest priority $|A|-1$, and tries to minimize the total weights of agents matched with their first preferences.
Then, \clingo considers the next highest priority $|A|-2$, and further tries to minimize the total weights of agents matched with their second preferences, and so on. In this way, \clingo finds a rank maximal stable matching.

\smallskip \noindent {\em Almost stable.}
Almost \sri aims to minimize the total number of blocking pairs for a matching~$M$. For that, we simply replace the hard constraint~(\ref{eq:P-stable}) that ensures stability, with the following weak constraints in our ASP formulation $P$ of \sri:
$$
\mathop{\leftarrow}^{\sim}  \ii{block}(x,y).\ [1@1,x,y] \quad (x\neq y)
$$

\noindent {\em Elaboration tolerance.} According to McCarthy~\citeyear{mccarthy98}, a representation is elaboration tolerant to the extent that it is convenient to modify a set of formulas expressed in the formalism to take into account new phenomena, and the simplest kind of elaboration is the addition of new formulas. In that sense, our representation $P$ of \sri is elaboration tolerant to variations of \sri, since the program $P$ is not changed at all (e.g., we add new rules to $P$ for Egalitarian \sri) or it is changed minimally (e.g., we replace a hard constraint by a weak constraint for Almost \sri).

\section{Experimental Evaluations}\label{sec:exp}

We have experimentally evaluated \srtiasp to understand its scalability over intractable \srti problems, and how it compares with two closely related methods over tractable \sr problems.

\smallskip \noindent {\em Setup.}
We have generated instances using the random instance generator that comes with \sricp.  It is based on the following idea~\cite{Mertens2005}: 1) generate a random graph ensemble $G(n,p)$ according to the Erdos-Renyi model~\cite{Erdos60}, where $n$ is the required number of agents and $p$ is the edge probability (i.e., each pair of vertices is connected independently with probability $p$); 2) since the edges characterize the acceptability relations, generate a random permutation of each agent’s acceptable partners to provide the preference lists. We define the {\em completeness degree} for an instance as the percentage $p*100$.

In our experiments, we have used \clingo (Version 5.2.2) on a machine with Intel Xeon(R) W-2155 3.30GHz CPU and 32GB RAM.

\begin{table}[t]
\caption{Scalability of \srtiasp in computation time, for  \sri, Egalitarian (E) \sri, Rank Maximal (R) \sri, and Almost (A) \sri.}
\label{tab:sri-opt}
\begin{minipage}{\textwidth}
\begin{tabular}{ccccccccc}
\hline \hline
            &       & \multicolumn{4}{c}{\sri}                        & E \sri    & R \sri    & A \sri \\
            &       & \#instances & average & \#instances & average   & average   & average   & average \\
completeness&       & with a      & time    & without any & time      & time      & time      & time \\
degree      & $|A|$ & solution    & (sec)   & solution    & (sec)     & (sec)     & (sec)     & (sec) \\
   \hline
25\%
 & 40 & 11 & 0.029 & 9 & 0.024 & 0.204 &  0.201 & 0.017   \\
 & 60 & 10 & 0.068 & 10 & 0.077 & 0.129 &  0.219 &  0.479   \\
 & 80 & 13 & 0.167 & 7 & 0.183 & 0.255 & 0.256 &  4.416  \\
 & 100 & 14 & 0.37 & 6 & 0.442 & 0.575 &  0.602 & 85.86 \\
 & 150 & 8 & 1.995 & 12 & 1.994 & 14.126  & 17.703 &  TO    \\
 & 200 & 10 & 8.794 & 10 & 8.323 & 59.18 &  83.85 &   TO  \\
\hline 
50\%
 & 40 & 11 & 0.06 & 9 & 0.084 & 0.106 &  0.108 & 0.136 \\
 & 60 & 16 & 0.276 & 4 & 0.381 & 0.535 & 0.56 & 3.748  \\
 & 80 & 13 & 0.852 & 7 & 1.106 & 1.82 & 1.85 &  3.748  \\
 & 100 & 12 &3.192 & 8 & 2.828 & 4.97 &  5.26 &  343.24  \\
 & 150 & 14 & 15.880 & 6 & 15.59 & 153.54 &  149.447 &   TO    \\
 & 200 & 9 & 69.65 & 11 & 72.08 & 524.58 & 704.3  &     TO  \\
\hline 
75\%
 & 40 & 14 & 0.136 & 6 & 0.207 & 0.722 & 0.33 & 0.486     \\
 & 60 & 13 & 0.744 & 7 & 0.997 & 8.171 & 1.85 & 15.526  \\
 & 80 & 8 & 3.971 & 12 & 3.346 & 6.99 &  7.16 &  144.80  \\
 & 100 & 13 & 11.06 & 7 & 8.801 & 19.04 &  20.65 & 885.39   \\
 & 150 & 9 & 50.520 & 11 & 51.700 & 492.7 & 529.16 &  TO  \\
 & 200 & 12 & 175.81 & 8 & 202.46 & 1757.0 &   1475.0 &  TO  \\
\hline 
100\%
 & 40 & 15 & 0.227 & 5 & 0.301 & 0.463 & 0.56 &  0.675   \\
 & 60 & 14 & 1.483 & 6 & 2.110 & 2.534 &  3.372 &  27.935   \\
 & 80 & 13 & 7.472 & 7 & 7.426 & 14.39 &  15.47 & 181.09 \\
 & 100 &10  & 24.43 & 10 & 16.268 & 35.92 &  40.42 &  2627.77 \\
 & 150 & 11 & 112.23 & 9 & 113.21 &360.7 &  362.12 &   TO  \\
 & 200 & 12 & 388.58 & 8 & 353.8 & 844.03 &  1147.0  &  TO  \\  \hline \hline
\end{tabular}%
 \\
 \footnotesize{TO: Timeout (over 3000 seconds)}
 \vspace{-.6\baselineskip}
\end{minipage}
\end{table}

\smallskip \noindent {\em Scalability of \srtiasp: \sri and its variations.}
We have generated instances of different sizes, where the number of agents are 20, 40, 60, 80,100, 150 and 200, and the completeness degrees are 25\%, 50\%, 75\% and 100\%. For each number of agents and for each completeness degree, we have generated 20 instances. We have experimented with these randomly generated instances to analyze the scalability of \srtiasp for \sri, Egalitarian \sri, Rank Maximal \sri, and Almost \sri. The results are shown in Table~\ref{tab:sri-opt}. We make the following observations from this table:

\vspace{-0.5ex}\begin{itemize}
\item[O1] The computation times for finding a stable matching (if one exists) and finding out that there exists no stable matching are comparable to each other.
\end{itemize} \vspace{-0.5ex}

\noindent Consider, for instance, the completeness degree 25\%, and 80 agents. For 13 (out of 20) instances, the average CPU time to compute a stable matching is 0.167 seconds. For the remaining 7 instances, the average CPU time to find that a stable matching does not exist is 0.183 seconds. These timings are comparable to each other.

\vspace{-0.5ex}\begin{itemize}
\item[O2] Computing an egalitarian stable matching generally takes slightly less time than computing a rank maximal stable matching.
\end{itemize} \vspace{-0.5ex}

\noindent For instance, for the 13 instances with a stable matching, computing egalitarian stable matchings takes on average 0.255 seconds; computing rank maximal stable matchings takes a similar amount of time, 0.256 seconds. For larger instances, we can observe that the latter takes a bit more time.

\vspace{-0.5ex}\begin{itemize}
\item[O3] Computing an almost stable matching significantly takes more time, compared to computing an egalitarian or a rank maximal stable matching.
    \vspace{-0.5ex}
\item[O4] Solving the optimization variants of \sri takes significantly more time, compared to solving \sri.
\end{itemize} \vspace{-0.5ex}

\noindent For instance, for the 7 instances without any stable matching, computing almost stable matchings takes in 4.416 seconds on average.

The observations O2 and O3 are interesting, considering all the three optimization variations of \sri are NP-hard. Though, the observation O4 (better illustrated in Fig.~\ref{fig:plot-sri}) is not surprising, considering that \sri is in P (Table~\ref{tab:complexity}).

We can can further observe the following about scalability:

\vspace{-0.5ex}\begin{itemize}
\item[O5] As the completeness degree increases, the computation times increase.
\vspace{-0.5ex}
\item[O6] As the number of agents increases, the computation times increase.
\end{itemize} \vspace{-1ex}

\begin{figure}
	\centering
    \resizebox{0.6\columnwidth}{!}{\includegraphics{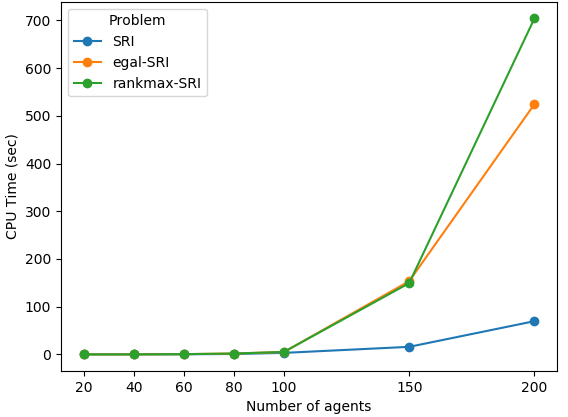}}
	\vspace{-.6\baselineskip}
	\caption{
Comparison of \sri with its optimization variants, when the completeness degree is 50\%.}
\label{fig:plot-sri}
	\vspace{-\baselineskip}
\end{figure}

\smallskip \noindent {\em Scalability of \srtiasp: \sri vs. \srti.}
To experiment with \srtiasp on \srti instances (under weak stability), we have randomly generated ties for the randomly generated \sri instances with the completeness degree $25\%$. For each agent $x$, we have 1) identified the set $T$ of agents that are not acceptable to $x$ and vice versa, and randomly picked one of these agents, say $y$, 2) identified the set $U$ of agents that are acceptable to $x$ and vice versa, and randomly picked one of these agents, say $z$, and 3) added $y$ in preference list of $x$ so that $x$ is indifferent between $y$ and~$z$. We have added ties as many as $25\%, 50\%, 75\%, 100\%$ of the number of agents. The results of these experiments are shown in Table~\ref{tab:sri-srti}. We can observe the following:

\vspace{-0.5ex}\begin{itemize}
\item[O7] Solving \srti takes significantly more time, compared to solving \sri.
\vspace{-0.5ex}
\item[O8] \sri instances that do not have any stable matching, often have stable matchings after ties are added.
\end{itemize} \vspace{-0.5ex}

\noindent The observation O7 is expected, since \srti is NP-complete whereas \sri is in P (Table~\ref{tab:complexity}).  O8 is reasonable since adding ties reduces the number of potential blocking pairs in general, and thus allows \srtiasp to explore more possibilities.

\begin{table}[t]
\centering
\caption{Scalability of \srtiasp\ (completeness degree of 25\%): \sri vs. \srti}
\label{tab:sri-srti}
\begin{minipage}{\textwidth}
\begin{tabular}{ccccccccccc}
\hline \hline
 & \multicolumn{2}{c}{\sri} & \multicolumn{8}{c}{\srti} \\
 & &  & \multicolumn{2}{c}{\%25 tie} & \multicolumn{2}{c}{\%50 tie} & \multicolumn{2}{c}{\%75 tie} & \multicolumn{2}{c}{\%100 tie} \\
 & \#inst. & avg  & \#inst.  & avg & \#inst. & avg & \#inst. & avg & \#inst. & avg \\
 & with & time  & with  & time & with & time & with & time & with & time \\
$|A|$ & solns. & (sec)  & solns.  & (sec) & solns. & (sec) & solns. & (sec) & solns. & (sec) \\
  \hline
 20 & 20 & 0.009 & 20 & 0.034 & 20 & 0.043 & 20 & 0.063 & 20 & 0.073 \\
 40 & 11 & 0.029 & 18 & 0.191 & 20 & 0.243 & 20 & 0.319 & 20 & 0.346 \\
 60 & 10 & 0.068 & 20 & 0.967 & 20 & 1.297 & 20 & 1.746 & 20 & 1.752 \\
 80 & 13 & 0.167 & 20 & 2.833 & 20 & 3.208 & 20 & 3.453 & 20 & 4.187 \\
 100 & 14 & 0.370 & 20 & 3.372 & 20 & 2.609 & 20 & 2.293 & 20 & 2.178  \\
 \hline \hline
\end{tabular}%

\vspace{-\baselineskip}
\end{minipage}
\end{table}

\begin{table}[ht]
\centering
\caption{\srtiasp vs. \sricp}
\label{tab:asp-cp}
\begin{tabular}{ccccccc}
\hline \hline
            &         & \#instances & \multicolumn{2}{c}{\srtiasp}           & \multicolumn{2}{c}{\sricp} \\
completeness&         & with a      & \multicolumn{2}{c}{average time (sec)} & \multicolumn{2}{c}{average time (sec)} \\
degree      &\#agents & solution    & exists solution & no solution          & exists solution & no solution \\

\hline
25\%
 & 80 & 13 & 0.167 & 0.183 & 0.318 & 0.310 \\
 & 100 & 14 & 0.370 & 0.442 & 0.389 & 0.492 \\
 & 150 & 8 & 1.995 & 1.994 & 0.704 & 0.697 \\
 & 200 & 10 & 8.794 & 8.323 & 1.047 & 0.999 \\
\hline 
50\%
 & 80 & 13 & 0.852 & 1.106 & 0.481 & 0.475 \\
 & 100 & 12 & 3.192 & 2.828 & 0.674 & 0.646 \\
 & 150 & 14 & 15.88 & 15.59 & 1.267 & 1.211 \\
 & 200 & 9 & 69.65 & 72.08 & 1.940 & 1.960 \\
\hline 
75\%
 & 80 & 8 & 3.971 & 3.346 & 0.684 & 0.682 \\
 & 100 & 13 & 11.062 & 8.801 & 0.967 & 0.961 \\
 & 150 & 9 & 50.52 & 51.70 & 1.795 & 1.780 \\
 & 200 & 12 & 175.81 & 202.46 & 3.214 & 3.263 \\
\hline 
100\%
 & 80 & 13 & 7.396 & 7.426 & 0.883 & 0.912 \\
 & 100 & 10 & 24.435 & 16.268 & 1.190 & 1.175 \\
 & 150 & 11 & 112.23 & 113.21 & 2.593 & 2.674 \\
 & 200 & 12 & 388.58 & 353.8 & 5.074  & 5.061 \\ \hline \hline
\end{tabular}%
\vspace{-0.5\baselineskip}
\vspace{-\baselineskip}
\end{table}

\smallskip \noindent {\em Scalability of \srtiasp vs. \sricp.}
We have experimented with \sricp~\cite{sricp2020}, which utilizes the CP solver \choco (Version 2.1.5), on the \sri instances generated by \sricp's random instance generator. The results for 80--200 agents are shown in Table~\ref{tab:asp-cp}.

\vspace{-0.5ex}\begin{itemize}
\item[O9] For large \sri instances, \sricp performs significantly better than \srtiasp.
\end{itemize} \vspace{-0.5ex}

\noindent
This observation has led to the following idea (mentioned by Prosser~\citeyear{Prosser2014}) for \sri instances that have stable matchings: ``Can we solve Egalitarian \sri faster than \srtiasp, by first enumerating all stable matchings using \sricp, and then finding the optimal one?''  We have noticed that the instances (generated by the random instance generator of \sricp) generally have one or two stable matchings. In that case, the answer to this question is Yes. This observation contradicts with the theoretical result on the NP-hardness of Egalitarian \sri. So we have generated some instances with more stable matchings.
For example, for an instance (sri90) with 90 agents and with more than 9 million stable matchings, \srtiasp takes 1.75 seconds to find an egalitarian stable matching whereas \sricp can not enumerate all these solutions (due to fast consumption of memory). For \sri instances with many stable matchings, it may be better to use \srtiasp to solve Egalitarian \sri; further investigations are planned as part of our future work.

Meanwhile, we have investigated a similar question for \sri instances that do not have any stable matching: ``Can we solve Almost \sri instances with $n$ agents faster than \srtiasp, by checking whether removing ${n \choose 2}, {n \choose 4},..$ agents (i.e., potential blocking pairs) leads to a stable matching?'' We have observed that, for small \sri instances with one blocking pairs, the answer to this question is Yes: If we remove two agents, then we can find a stable matching. For larger instances with many blocking pairs, the answer is negative. For example, for an instance (sri60a) with 60 agents, that does not have any stable matching, we have observed that \srtiasp finds an almost stable matching with 10 blocking pairs in 9.057 seconds. With the enumerate-test method mentioned in the question, we have to enumerate ${60 \choose 2} + {60 \choose 4} + ... + {60 \choose 10}$ (more than $7{\times}10^{10}$) instances, and check them one by one using \sricp until an almost stable matching is found. Assuming that \sricp takes 0.001 seconds per instance, in the worst case we will have to test all $7{\times}10^{10}$ instances, and it will take at least 2 years.  This observation confirms with the theoretical result on the NP-hardness of Almost \sri; further investigations are planned as part of our future work.

\begin{table}[t]
\centering
\caption{\srtiasp vs. \sriaf}
\label{tab:asp-argsri}
\begin{tabular}{ccccccc}
\hline \hline
            &         & \#instances & \multicolumn{2}{c}{\srtiasp}           & \multicolumn{2}{c}{\sraf} \\
completeness&         & with a      & \multicolumn{2}{c}{average time (sec)} & \multicolumn{2}{c}{average time (sec)} \\
degree      &\#agents & solution    & exists solution & no solution          & exists solution & no solution \\

\hline
\hline
25\%
 & 80 & 13 & 0.167 & 0.183 & 0.118 & 0.126 \\
 & 100 & 14 & 0.370 & 0.442 & 0.242 & 0.254 \\
 & 150 & 8 & 1.995 & 1.994 & 1.002 & 1.077 \\
 & 200 & 10 & 8.794 & 8.323 & 2.674 & 2.604 \\
\hline 
50\%
 & 80 & 13 & 0.852 & 1.106 & 0.521 & 0.527 \\
 & 100 & 12 & 3.192 & 2.828 & 1.073 & 1.01 \\
 & 150 & 14 & 15.88 & 15.59 & 5.029 & 4.918 \\
 & 200 & 9 & 69.65 & 72.08 & 13.35 & 14.19 \\
\hline 
75\%
 & 80 & 8 & 3.971 & 3.346 & 1.248 & 1.252 \\
 & 100 & 13 & 11.062 & 8.801 & 2.795 & 2.772 \\
 & 150 & 9 & 50.52 & 51.70 & 12.65 & 12.222 \\
 & 200 & 12 & 175.81 & 202.46 & 36.35 & 35.72 \\
\hline 
100\%
 & 80 & 13 & 7.396 & 7.426 & 2.525 & 2.523 \\
 & 100 & 10 & 24.435 & 16.268 & 5.677 & 5.479 \\
 & 150 & 11 & 112.23 & 113.21 & 24.94 & 24.75 \\
 & 200 & 12 & 388.58 & 353.8 & 77.89 & 74.67 \\ \hline \hline
\end{tabular}%
\vspace{-0.5\baselineskip}
\vspace{-\baselineskip}
\end{table}

\smallskip \noindent {\em Scalability of \srtiasp vs. \sraf}
The ASP-based method \sraf~\cite{Amendola18} utilizes abstract argumentation frameworks~\cite{DUNG1995} for the Stable Marriage problem.
According to \sraf, an argumentation framework $\ii{AF}{=}(\ii{Arg},\ii{Att})$ models an \sr instance if the arguments in $\ii{Arg}$ are pairs of different agents, and the attacks $((a,b),(x,y))$ in $\ii{Att} \subseteq \ii{Arg} {\times} \ii{Arg}$ satisfy the following properties: (i) $x=a$ and $b \prec_{x} y$, or (ii) $y=b$ and $a \prec_{y} x$. For every \sr instance, once the arguments and attacks are generated, they are translated into a program $P_{\is{AF}}$ using the existing methods~\cite{WuCG09}. In particular, for every argument $(a,b)$ in $\ii{AF}$, if $(x_1,y_1), (x_2,y_2), \dots, (x_m,y_m)$ are the arguments that attack $(a,b)$, the following rule is included in $P_{\is{AF}}$:
$$in(a,b) \lar \no\ in(x_1,y_1), \no\ in(x_2,y_2), \dots, \no\ in(x_m,y_m).$$
There is a one-to-one correspondence between the answer sets for the program $P_{\is{AF}}$ and the stable extensions of $\ii{AF}$, due to Theorem~2 of Amendola~\citeyear{Amendola18}.

We have extended \sraf's argumentation framework from \sr to \sri, by defining the arguments as pairs of agents that are acceptable to each other. implemented (in Python) the transformation of an \sri instance into an argumentation framework $\ii{AF}$ and then to a program $P_{\is{AF}}$. We have experimented with this extended \sraf (called \sriaf from now on) on the \sr instances generated by \sricp's random instance generator. The results
are shown in Table~\ref{tab:asp-argsri}.

\vspace{-0.5ex}\begin{itemize}
\item[O10] For large \sri instances, \sriaf performs significantly better than \srtiasp.
\end{itemize} \vspace{-1ex}

\section{Conclusion} \label{sec:conc}

We have developed a formal framework, called \srtiasp, that is sound and complete (Theorem~\ref{thm:correct}) and general enough to provide solutions to various stable roommates problems, such as, \sr, \sri, \srt, \srti, Egalitarian \srti, Rank Maximal \srti, Almost \srti. Except for \sr and \sri, all these variations are intractable (Table~\ref{tab:complexity}).
Our ongoing work involves extending \srtiasp to other computationally hard stable roommates problems.

Since \srtiasp utilizes Answer Set Programming (ASP), the formulations of problems are concise and elaboration tolerant, and thus \srtiasp provides a flexible framework to study variations of stable roommates problems. Having such a flexible framework and implementation is valuable for studies in matching theory.

We have evaluated \srtiasp over different sizes of randomly generated \sri instances, and have made many interesting observations (O1--O10) about its scalability over different intractable variations of \sri, and in comparison with \sricp and \sraf over tractable variations of \sr.
The results of our empirical analysis of \srtiasp are promising, in particular, for computationally hard problems. Considering that the input sizes of the instances are large enough for many dormitories, the results of experiments are also promising for real-world applications.

Comparisons with \sricp and \sriaf has helped us to better observe the flexibility of \srtiasp due to elaboration tolerant ASP representations. It is easier to extend \srtiasp to address different variations of \sr, while \sricp and \sriaf require further studies in modeling as well as implementation. Note that \sricp uses \choco via a Java wrapper, and \sriaf solves \sri via an argumentation framework. As a future work, we plan to investigate how \sricp and \sriaf can be extended to \srti and its intractable versions.

The Stable Marriage problem with Ties (\smt) under strong stability, which can be solved with a polynomial time algorithm~\cite{irving1994}, has been used as a benchmark in ASP competitions. Its representation is based on ranks instead of preferences, and does not utilize choice rules, cardinality expressions or weak constraints. With its intractable variations (under weak stability), \srtiasp contributes to ASP studies by providing an elaboration tolerant formulation and a complementary and rich set of benchmark instances of Stable Roommates problems.

\bibliographystyle{acmtrans}


\label{lastpage}
\end{document}